\relax
\documentclass[letterpaper]{article}
\usepackage[final]{nips_2016} 
\usepackage{times}
\usepackage{helvet}
\usepackage{courier}
\usepackage{amssymb}
\usepackage{graphicx}
\usepackage[lofdepth, lotdepth]{subfig}
\usepackage{algorithm}
\usepackage{algorithmic}
\newcommand{\learner}{\mathcal{L}}
\newcommand{\target}{c^*}
\newcommand{\scriptC}{\mathcal{C}}

\newcommand{\scriptF}{\mathcal{F}}
\newcommand{\real}{\mathbb{R}}
{\vspace{0.1in}}
\newtheorem{proposition}{\vspace{0.1in}{\noindent \bf Proposition}}{\vspace{0.1in}}
\newtheorem{lemma}{\vspace{0.1in}{\noindent \bf Lemma}}{\vspace{0.1in}}
\newenvironment{proof}{\vspace{0.1in}{\noindent \bf Proof}}{\vspace{0.1in}}

\begin{document}
%
\title{Analysis of a Design Pattern for Teaching with Features and Labels}
\author{Christopher Meek$^\dagger$, Patrice Simard$^\dagger$, Xiaojin Zhu$^+$ \\ Microsoft Research$^\dagger$ and University of Wisconsin-Madison$^+$}
\maketitle
\begin{abstract}
We study the task of teaching a machine to classify objects using features and labels. We introduce the Error-Driven-Featuring design pattern for teaching using features and labels in which a teacher prefers to introduce features only if they are needed. We analyze the potential risks and benefits of this teaching pattern through the use of teaching protocols, illustrative examples, and by providing bounds on the effort required for an optimal machine teacher using a linear learning algorithm, the most commonly used type of learners in interactive machine learning systems. Our analysis provides a deeper understanding of potential trade-offs of using different learning algorithms and between the effort required for featuring and labeling.
\end{abstract}

\section{Introduction}
Featuring and labeling are critical parts of the interactive machine learning process in which a person and a machine learning algorithm coordinate to build a predictive system (a classifier, entity extractor, etc.). Unlike the case of using labels alone, little is known about how to quantify the effort required to teach a machine using both features and labels. In this paper, we consider the problem of teaching a machine how to classify objects when the teacher can provide labels for objects and provide features---functions from objects to values. Our aim is to understand the effort required by a teacher to find a suitable representation for objects, to teach the target classification function, and to provide guidance to teachers about how to provide features and labels when teaching.

Similar to previous work on active learning and teaching dimension, we take an idealized view of the cost of labeling and featuring. In particular, we ignore variability in the effort required for these respective actions. In addition, similar to the work on teaching dimension, we assume an idealized teacher with complete knowledge about the learner, target classification function and the range of possible objects that we want to classify.

We analyze the effort required to teach a classification function relative to a given set of feature functions. This set of features functions can be thought of as a set of teachable functions. There are several observations that motivate us to quantify teaching effort relative to a set of feature functions. It is natural to expect that the available set of teachable functions depends on the specific learner that we are teaching and the  types of objects that we want to classify (e.g., images versus text documents). In addition, the teaching effort required to teach a learner is heavily dependent on the available set of functions. For instance, if the teacher could directly teach the learner the target classification function then only one function would be required, and, for wide variety of learning algorithms, the teacher would only be require to provide two labeled examples in the case of binary classification. Of course, it is unreasonable to expect that the target classification function can directly be encoded in a feature function and, in fact, if this is possible then we need not use a machine learning algorithm to build the predictor. For these reasons, we assume that there is a set of features that are teachable and define the effort relative to this set of features. In order to capture dependencies among features we consider a lattice of sets of features rather than a set of features. We use the lattice to enforce our assumption that features are taught one at a time and to capture other dependencies such as only allowing features to be taught if all of the constituent features have been taught (e.g., the feature of {\em tall and heavy} can only be included in a feature set if the features of being {\em tall} and of being {\em heavy} have previously been defined). Thus, the lattice of feature sets captures the potential alternative sequences of features that the teacher can use to teach a learner.

We introduce the Error-Driven-Featuring (EDF) design pattern for teaching in which the teacher prefers to add features only if they are needed to fix a prediction error on the training set.
In order to analyze the risks and benefits of the EDF teaching pattern we consider two teaching protocols, one which forces the teacher to use the EDF teaching pattern and the other which does not. By quantifying the featuring and labeling effort required by these protocols we can provide a deeper understanding of the  risks and benefits of the EDF pattern and potential trade-offs between featuring and labeling more generally. In our analysis we consider two specific learning algorithms; a one-nearest-neighbor classifier and a linear classifier.
Using our measures of teaching cost we demonstrate that there are significant risks of adding features for high-capacity learning algorithms (1NN) which can be controlled by using a low-capacity learning algorithm (linear classifier). We also demonstrate that the additional labeling costs associated with using the EDF teaching pattern for both high and low capacity learning algorithms can be bounded. The combination of these results suggest that it would be valuable to empirically evaluate the EFT design pattern for teaching. In analyzing the costs of the Error-Driven-Featuring protocol we provide new results on the hypothesis specific pool-based teaching dimension of linear classifiers and pool-based exclusion dimension of linear classifiers. \footnote{This paper is an extended version of the paper by Meek et al (2016). \nocite{MeekSimardZhu2016designpatterns}}

\section{Related Work}

There has been a a variety of work aimed at understanding the labeling effort required to build classifiers. In this section we briefly review related work.  First we note that this work shares a common roots with the work of Meek (2016) but there the focus is on prediction errors rather than teaching effort. \nocite{Meek2016errors}

One closely related concept is that of teaching dimension. The primary aim of this work is to quantify the worst case minimal effort to teach a learner one classification function (typically called a concept in this literature) from among a set of alternative classification functions. There is a large body of work aimed at understanding the teaching dimension, refining teaching dimension (e.g., extended, recursive) and the relationship between these and other concepts from learning theory such as the VC-dimension (e.g., Doliwa et al 2014, Balbach 2008, Zilles et al 2011). \nocite{JMLR:v15:doliwa14a, Balbach:2008:MTU:1365093.1365255,Zilles2011Models}  Our work, rather than attempting to quantify the difficulty of learning among a set of classifications, is aimed at quantifying the effort required to teach any particular classification function and to understand the relationship between adding features and adding labels. The work on teaching dimension abstracts the role of the learner and rather deals directly with hypothesis classes of classification functions. Furthermore, the work on teaching dimension abstracts away the concept of features making it useless for understanding the interplay between learner, featuring and labeling.  That said, several of the concepts that we use have been treated previously in this and related literature. For instance, the idea of a concept teaching set is closely related to that of a teaching sequence (Goldman and Kearns 1995)\nocite{Goldman1995Complexity} and our optimal concept specification cost is essentially the specification number of a hypothesis (Anthony et al 1992)\nocite{Anthony:1992:ESE:130385.130420}; we add concept to distinguish it from representation specification cost.  Other existing concepts include the exclusion dimension (Angluin 1994) \nocite{Angluin:2004:QR:982360.982362} and the unique specification dimension (Hedigus 1995) \nocite{Hegedus:1995:GTD:225298.225311} and the certificate size (Hellerstein et al 1996) which are similar to our invalidation cost. \nocite{Hellerstein:1996:MQN:234752.234755} In addition, Liu et al (2016) \nocite{Liu2016ICML} define the teaching dimension of a hypothesis which is equivalent to the specification number and our concept specification cost. They also provide bounds on the concept specification cost for linear classifiers. Their results are related to our Proposition~\ref{thm:LinConceptSpecBound} but, unlike our result, assume that the space of objects is dense. In the terms of Zhu (2015), \nocite{Zhu2015Machine} we provide the hypothesis specific teaching dimension for pool-based teaching. For many domains such as image classification, document classification and entity extraction and associated feature sets the assumption of a dense representation is unnatural (e.g., we cannot have a fractional number of words in a document). Like other work on classical teaching dimension, this work does not consider teaching with both labels and features.

The other body of related work is active learning.  
The aim of this body of work is to develop algorithms to choose which items to label and the quality of an algorithm is measured by the number of labels that are required to obtain a desirable classification function. Thus, given our interest on both labeling and featuring this body of work is perhaps better named ``active labeling''. In contrast to the work on teaching dimension where the teacher has access to target classification function, in active learning, the teacher must choose the item to label without knowledge of the target classification function. This makes active learning critical to many practical systems. An excellent survey of research in this area is given by Settles (2012).\nocite{settles.book12} Not surprisingly, the work on active learning is related to work on teaching dimension (Hanneke 2007).\nocite{Hanneke2007Teaching}

\section{Features, Labels and Learning Algorithms}

In this section, we define features, labels and learning algorithms. These three concepts are the core concepts needed to discuss the cost of teaching a machine to classify objects. Thus, these definitions are the foundation of the remainder of the paper. In addition to providing these definitions, we also describe two properties of learning algorithms related to machine teaching and we describe two specific learning algorithms that are used in the remainder of the paper. 

We are interested in building a classifier of objects. We use $x$ and $x_i$ to denote particular objects and $X$ to denote the set of objects of interest. We use $y$ and $y_i$ for particular labels and $Y$ to denote the space of possible labels. For binary classification $Y=\{0,1\}$.  A classification function is a function from $X$ to $Y$.\footnote{Note that, while we call this mapping a classification function, the definition encompasses a broad class of prediction problems including structured prediction, entity extraction, and regression.} The set of classification functions is denoted by $\scriptC = X\rightarrow Y$. We use $\target$ to denote the target classification function.

Central to this paper are features or functions which map objects to scalar values. A {\em feature} $f_i$ (or $g_i$) is a function from objects to real numbers (i.e. $f_i \in X\rightarrow \real$). A {\em feature set} is a set of features and we use $F,F_i$ and $G_i$ to denote generic feature sets. The feature set $F_i=\{f_{i,1},\ldots,f_{i,p}\}$ is $p$-dimensional. We use a $p$-dimensional feature set to map an object to a point in $\real^p$. We denote the mapped object $x_k$ using feature set $F_i$ by $F_i (x_k )=(f_{i,1} (x_k ),\ldots,f_{i,p} (x_k ))$ where the result is a vector of length $p$ where the $j^{th}$ entry is the result of applying the $j^{th}$ feature function in $F_i$ to the object. 

We define the potential sequences of teachable features via a lattice of feature sets. Our definition of a feature lattice enforces the restriction that features are taught sequentially. We use $R=\{f_1,f_2, f_3, \ldots \}$ to denote the set of all teachable feature functions for a set of objects $X$. A feature lattice $\scriptF$ for a feature set $R$ is a set of finite subsets of $R$ (thus $\scriptF\subseteq 2^R$) such that if $F_i\in\scriptF$ then either $F_i=\emptyset$ or there is a $F_j\in \scriptF$ such that $F_j\subset F_i$  and $|F_j |+1=|F_i |$. We restrict attention to finite sets to capture the fact that teachers can only teach a finite number of features. 
We note that the feature lattice also allows us to represent constraints on the order in which features can be taught. Such constraints arise naturally. For instance, before teaching the concept of the area of a rectangle one needs to first teach the concepts of length and width (e.g., feature $f_3(x)=f_1(x)\times f_2(x)$ can be added only if both $f_1$ and $f_2$ have been added as features).

These definitions are illustrated in Figure~\ref{fig:TeachingEffort}.

\begin{figure}
\subfloat[][] 
{
\includegraphics[width=2.72in]{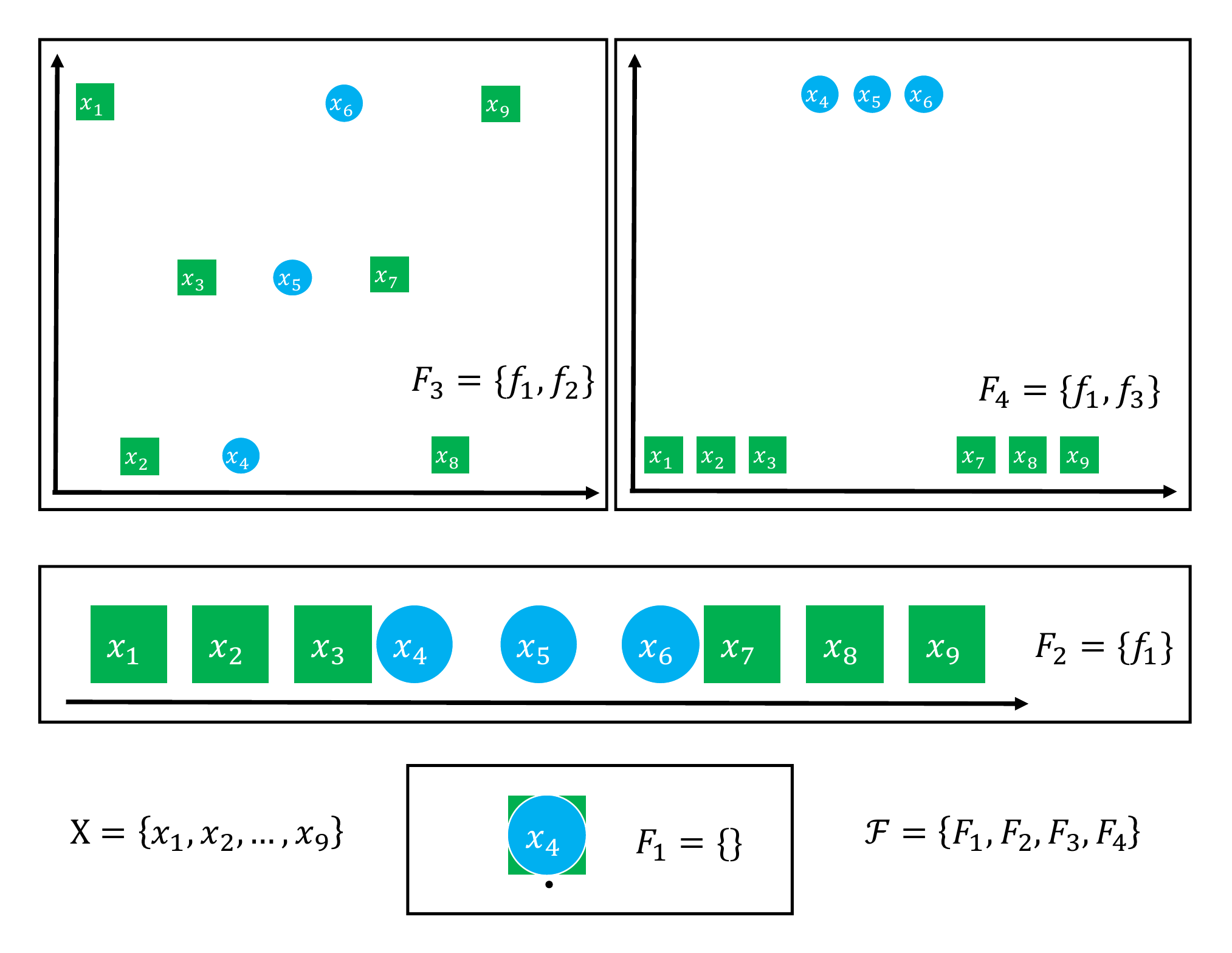}
\label{fig:TeachingEffort1}
}
\subfloat[][]
{
\includegraphics[width=2.72in]{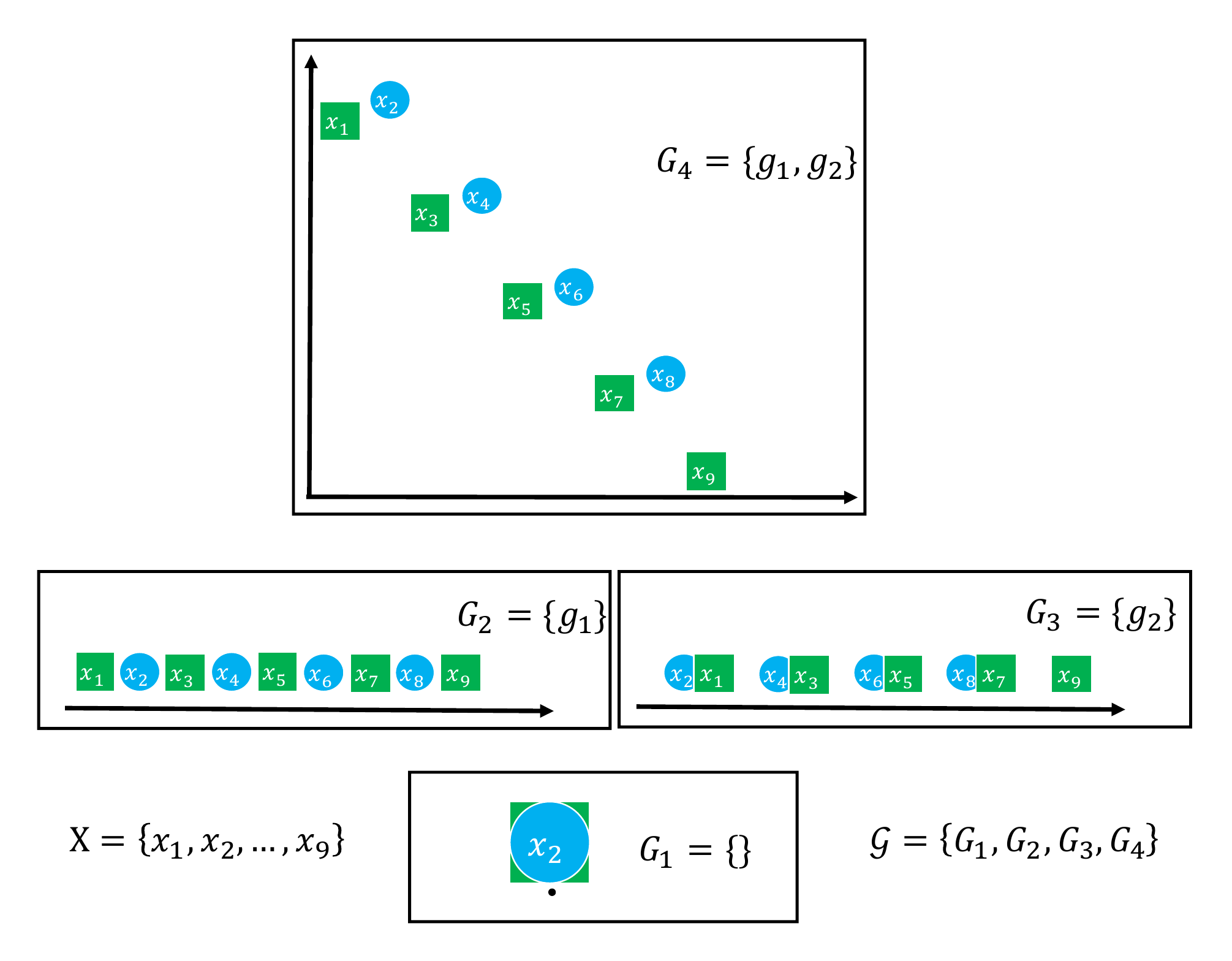}
\label{fig:TeachingEffort2}
}
\caption{Two example feature lattices each with four feature sets and nine objects. The shape and color of the objects denote the target binary classification. Each rectangular region is associated with a $d$-dimensional feature set and contains a plot of the objects in $\real^d$. The lowest rectangular region in each panel is associated with an empty feature set which maps all objects to the same point. We graphically depict this mapping by overlaying the circles on top of the rectangles.}
\label{fig:TeachingEffort}
\end{figure}

In order to define a learning algorithm we first define training sets and, because we are considering learning with alternative feature sets, featurized training sets. 
A training set $T\subset X\times Y$ is a set of labeled examples. We consider only honest training sets, that is, $T\subset X\times Y$ such that $\forall (x,y)\in T$ it is the case that $\target(x)=y$.
We say that the training set $T$ has $n$ examples if $|T|=n$ and denote the set of training examples as $\{(x_1,y_1),\ldots,(x_n,y_n)\}$. A training set is unfeaturized. We use feature sets to create featurized training sets.  For $p$-dimensional feature set $F_i$ and an $n$ example training set $T$ we denote the featurized training set $F_i(T)= \{(F_i (x_1 ),y_1 ),\ldots,(F_i (x_n ),y_n)\}\in \{\real^p\times Y\}^n$. We call the resulting training set an $F_i$ featurized training set or an $F_i$ featurization of training set $T$.

Now we are prepared to define a learning algorithm. First, a $d$-dimensional learning algorithm $\ell_d$ is a function that takes a $p$-dimensional feature set $F$ and a training set $T$ and outputs a function $h_p\in \real^p\rightarrow Y$. Thus, the output $h_p$ of a learning algorithm using $F_i$ and training set $T$ can be composed with the functions in the feature set to yield a classification function of objects (i.e., $h_p\circ F_i\in \scriptC$). The hypothesis space of a $d$-dimensional learning algorithm $\ell_d$ is the image of the function $\ell_d$ and is denoted by $H_{\ell_d}$ (or $H_d$ if there is no risk of confusion). A classification function $c\in \scriptC$ is {\em consistent with a training set} $T$ if $\forall (x,y)\in T$ it is the case that $c(x)=y$. A $d$-dimensional learning algorithm $\ell_d$ is {\em consistent} if the learning algorithm outputs a hypothesis consistent with the training set whenever there is a hypothesis in $H_d$ that is consistent with the training set. A {\em vector learning algorithm} $\ell=\{\ell_0,\ell_1,\ldots \}$ is a set of $d$-dimensional learning algorithms one for each dimensionality.  A {\em consistent} vector learning algorithm is one in which each of the $d$-dimensional learning algorithms is consistent. Finally, a {\em (feature-vector) learning algorithm} $\learner$ takes a feature set $F$, a training set $T$, and a vector learning algorithm $\ell$ and returns a classification function in $\scriptC$. In particular $\learner_\ell(F,T)=\ell_{|F|}(F,T)\circ F\in \scriptC$. When the vector learning algorithm is clear from context or we are discussing a generic vector learning algorithm we drop the $\ell$ and write $\learner(F,T)$.

One important property of a feature set is whether it is sufficient to teach the target classification function $\target$.
A feature feature set $F$ is {\em sufficient} for learner $\learner$ and target classification function $\target$ if there exists a training set $T$ such that $\learner(F,T)=\target$.

A natural desiderata of a learning algorithm is that adding a feature to a sufficient feature set should not make it impossible to teach a target classification function. We capture this with the following property of a learning algorithm.
We say that a learning algorithm $\learner$ is {\em monotonically sufficient} if it is the case that if $F$ is sufficient then $F'\supset F$ is sufficient. Many learning algorithms, in fact, have this property. 

We distinguish two type of training sets that are central to teaching.
First, a training set $T$ is a {\em concept teaching set} for feature set $F$ and learning algorithm $\learner$ if $\learner(F,T)=\target$. Second, a training set $T$ an {\em invalidation set} if there is an example $(x,y)\in T$ that is not correctly classified by $\learner(F_i,T)$.

The following proposition demonstrates that, for consistent learning algorithms, finding an invalidation set demonstrates that a feature set is not sufficient for the target classification function.

\begin{proposition}$\!\!\!$ \label{thm:insufficient}
If learning algorithm $\learner$ is consistent and $T$ is an invalidation set for feature set $F_i$, target concept $\target$, and $\learner$ then $F_i$ is not sufficient for $\target$ and $\learner$.
\end{proposition}

Meek (2016) suggests that identifying minimal invalidation sets might be helpful for teachers wanting to identify mislabeling errors and representation errors. In this paper, an invalidation set is an indication of a representation errors because we assume that the labels in the training set are correct implying that there are no mislabeling errors.

In the remainder of the paper we consider two binary classification algorithms ($Y=\{0,1\}$). 
The first learning algorithm is a consistent one-nearest-neighbor learning algorithm $\learner_{1NN}$. 
Our one-nearest-neighbor algorithm is a set of $d$-dimensional one-nearest-neighbor learning algorithms that use a $d$-dimensional feature set to project the training set into $\real^d$. The algorithm identifies the set of closest points and outputs the minimal label value of points in that set. Thus, if there is more than one closest point and their labels disagree then the learned classification will output 0. By construction, this is a consistent learning algorithm.

The second learning algorithm is a linear learning algorithm $\learner_{lin}$.
Our consistent linear learning algorithm is a set of $d$-dimensional linear learning algorithms for which the decision surface is defined by a hyperplane in the $\real^d$ or, more formally, by $c(x|F,w,b)=sign(w\cdot F(x) + b)$  where the hyperplane is defined in terms of weights $(w,b)$. We consider the linear learner $\learner_{lin}$ that produces the maximum margin separating hyperplane for a training set when one exists and outputs the constant zero function otherwise. Note that the maximum margin separating hyperplane for a training set is the separating hyperplane that maximizes the minimum distance between points in the training set and the hyperplane. again, by construction, this is a consistent learning algorithm.

Note that we say that a feature set $F$ is linearly sufficient for the target classification function if $F$ is sufficient for the target classification function when using a consistent linear learning algorithm.

We finish this section with the following proposition that demonstrates our learning algorithms are both monotonically sufficient.
\begin{proposition} \label{thm:sufficiency}
The learning algorithms $\learner_{1NN}$ and $\learner_{lin}$ are monotonically sufficient.
\end{proposition}

\section{Teaching Patterns, Protocols and Costs} \label{sec:protocols}

In this section, we introduce our Error-Drive-Featuring (EDF) design pattern for teaching and two teaching protocols. We introduce the teaching protocols as a means to study the risks and benefits of our EDF teaching pattern. 

Teaching patterns are related to design patterns (Gamma et al 1995). \nocite{gamma1995design} Whereas design patterns for programming are formalized best practices that a programmer can use to design software solutions to common problems, a design pattern for teaching (or teaching pattern) is a formalized best practice that a teacher can use to teach a computer. 

We use a pair of teaching protocols to study the risks and benefits of our EDF teaching pattern. A teaching protocol is an algorithmic description of a method by which a teacher teaches a learner. In order to study a teaching pattern, in one protocol, we force the teacher to follow the teaching pattern and, in the other, we allow the teacher full control over their actions.  

We contrast our teaching protocols by comparing the optimal teaching costs and, in a subsequent section, bounds on optimal teaching costs. To facilitate the discussion of optimal teaching costs we next define several teaching costs associated with a feature set.

\subsection{Optimal Feature Set Teaching Costs}

Next we define a set of costs for a feature set.
The first measure is a measure of the cost of specifying the feature set. 
We measure the {\em representation specification cost} of a feature set $F$ by the cardinality of the feature set $|F_i|$. This idealized measure does not differentiate the effort required to specify features. In practice, different features might require different effort to specify and the cost to specify different features will depend upon the interface through which features are communicated to the learner.

The second measure of a feature set is a measure of the cost of specifying a target classification function using the feature set and a given learning algorithm. We measure the {\em optimal concept specification cost} by the size of the minimal concept teaching set for $F$ using learner $\learner$ if $F$ is sufficient and to be infinite otherwise. 

The third measure of a feature set is a measure of the cost of demonstrating that the feature set is not sufficient for a given learning algorithm. We measure the {\em optimal invalidation cost} of a feature set $F$ using learner $\learner$ by the size of the minimal invalidation set if $F$ is not sufficient and infinite otherwise.

We define the optimal feature set cost vector $FSCost(F,\learner)$ for a feature set $F$ and learning algorithm $\learner$. The feature set cost vector is of length three where the first component is the feature specification cost, the second component is the optimal concept specification cost and the third component is the optimal invalidation cost. 

Consider the feature set $F_2$ in Figure~\ref{fig:TeachingEffort1}.
The training set with three objects $T=\{x_2, x_5$, $x_8\}$ is a minimal concept teaching set for $F_2$ and a minimal invalidation set for $F_1$. Thus, we can now specify the optimal feature set costs for $F_2$: the representation specification cost is $|F_2|=1$, the optimal concept specification cost is $|T|=3$, the optimal invalidation cost is $\infty$ (i.e., $FSCost(F_2,\learner_{1NN})=(1,3,\infty)$). The optimal feature set cost vectors for other feature sets are shown in Table~\ref{tab:fscosts}.

\subsection{Analysis of Teaching Protocols}

\begin{figure*}[t]
\begin{tabular*}{\textwidth}{cc}
\begin{minipage}[t]{0.46\textwidth}
\begin{algorithm}[H]
\caption{Open-Featuring}
\label{alg:open}
\begin{algorithmic}
\STATE{\bf Input} learning algorithm $\learner$, set of objects $X$, a feature lattice $\scriptF$, and target classification function $\target$.
\STATE $T=\{\}$ \hspace{0.5in} // training set $T\subset X\times Y$
\STATE $F=\{\}$ \hspace{0.5in} // feature set $F\in \scriptF$
\STATE $c =\learner(T,F)$;
\WHILE {$\exists x\in X$ such that $c(x)\neq \target(x)$}
\STATE $Action =$ \underline{Choose-action}$(F,T,\learner)$;
\IF{(Action == "Add-feature")}
\STATE $PossFeat = \bigcup_i \{F_i \in \scriptF s.t. |F_i| = |F|+1\}\setminus F $
\STATE $f =$ \underline{Add-feature}(PossFeat);
\STATE $F = F \cup \{ f \}$
\STATE $c = \learner(T,F)$;
\ELSE 
\STATE $(x,y) =$ \underline {Add-example}($X,F,T,\learner$)
\STATE $T=T\cup (x,y)$;
\STATE $c=\learner(T,F)$;
\ENDIF
\ENDWHILE
\STATE return c;
\end{algorithmic}
\end{algorithm}
\end{minipage}
&\hspace{0.03\textwidth}
\begin{minipage}[t]{0.46\textwidth}
\begin{algorithm}[H]
\caption{Error-Driven-Featuring}
\label{alg:error}
\begin{algorithmic}
\STATE{\bf Input} learning algorithm $\learner$, set of objects $X$, a feature lattice $\scriptF$, and target classification function $\target$.
\STATE $T=\{\}$ \hspace{0.5in} // training set $T\subset X\times Y$
\STATE $F=\{\}$ \hspace{0.5in} // feature set $F\in \scriptF$
\STATE $c=\learner(T,F)$;
\WHILE {$\exists x\in X$ such that $c(x)\neq \target(x)$}
\STATE $(x,y) =$ \underline{ Add-example}($X,F,T,\learner$);
\STATE $T=T\cup (x,y)$;
\STATE $c=\learner(T,F)$;
\WHILE{($\exists (x,y)\in T$ such that $c(x)\neq y$)}
\STATE $PossFeat = \bigcup_i \{F_i \in \scriptF s.t. |F_i| = |F|+1\}\setminus F $
\STATE $f =$ \underline{Add-feature}(PossFeat);
\STATE $F = F \cup \{ f \}$
\STATE c = $\learner(T,F)$;
\ENDWHILE
\ENDWHILE
\STATE return c;
\end{algorithmic}
\end{algorithm}
\end{minipage}
\end{tabular*}
\caption{Algorithms for two teaching protocols; Open-Featuring and Error-Driven-Featuring.}
\label{fig:teachingprotocols}
\end{figure*}

Figure~\ref{fig:teachingprotocols} describes two teaching protocols. In Algorithm~\ref{alg:open}, the teacher is able to choose whether to add a feature or to add a labeled example. Because the teacher can choose when to add a feature and when to add a labeled example (i.e., the teacher implements the Choose-action function) we call this teaching protocol the Open-Featuring protocol. 
When adding a feature (the Add-feature function), the teacher selects one of the features that can be taught given the feature lattice $\scriptF$ and the teaching protocol adds the feature to the current feature set and retrain the current classifier. When adding a label (the Add-example function), the teacher chooses which labeled example to add to the current training set and the teaching protocol adds the example to the training set and retrains the current classifier.

In Algorithm~\ref{alg:error}, the teacher can only add a feature if there is a prediction error in the training set. From Proposition~\ref{thm:insufficient}, if we are using a consistent learner we know that this implies that the feature set is not sufficient and indicates the need to add additional features. Note this assumes that the teacher provides correct labels. For a related but alternative teaching protocol that allows for mislabeling errors see Meek (2016). \nocite{Meek2016errors} In this protocol, if the current feature set is not sufficient, a teacher adds labeled examples to find an invalidation set which then enables them to add a feature to improve the feature representation. This process of creating invalidation sets continues until a sufficient feature set is identified. An ideal teacher under this protocol would want to minimize the effort to invalidate feature sets that are not sufficient. The cost of doing this for a particular feature set can be measured by the invalidation cost. There is a possibility that one can reuse examples from the invalidation sets of previously visited smaller feature sets, but the sum of the invalidation costs along paths in the feature lattice provides an upper bound on the cost of discovering sufficient feature sets.

Given these two protocols is natural to compare costs by the number of features added and the number of labeled examples that are added in defining the classifier. We can then associate a {\em teaching cost} with each feature set in the feature lattice $\scriptF$. The teaching cost is also a function of the learning algorithm, and the featuring protocol (Open or Error-driven). The optimal teaching costs for $\learner_{lin}$ and $\learner_{1NN}$ for different feature sets is given in Table~\ref{tab:teachingcosts}. An infinite label cost indicates that the feature set cannot be used to teach the target classification function using that protocol and learning algorithm. 
Because our teaching cost has two components, we would need to choose method to combine these two quantities in order to discuss optimal teaching policies. Once the teacher has provided the learner a feature set that is sufficient the teacher needs to teach the concept represented by the $\target$ classification function. The labeling cost required to do this is captured by the concept specification cost.

\begin{table}[ht]
\centering
\subfloat[Subtable 1 list of tables text][Feature set costs for using $\learner_{1NN}$ and $\learner_{lin}$.]
{
\begin{tabular}{c|cc}\label{tab:fscosts}

Feat. Set & $\learner_{1NN}$ & $\learner_{lin}$\\
\hline
$F_1$ & (0,$\infty$,2) & (0,$\infty$,2) \\
$F_2$ & (1,3,$\infty$) & (1,$\infty$,3) \\
$F_3$ & (2,7,$\infty$) & (2,$\infty$,3) \\
$F_4$ & (2,2,$\infty$) & (2,2,$\infty$) \\
\hline
$G_1$ & (0,$\infty$,2) & (0,$\infty$,2) \\
$G_2$ & (1,9,$\infty$) & (1,$\infty$,3) \\
$G_3$ & (1,8,$\infty$) & (1,$\infty$,3) \\
$G_4$ & (2,9,$\infty$) & (2,3,$\infty$) \\
\end{tabular}
}
\qquad
\subfloat[Subtable 1 list of tables text][Optimal teaching costs using $\learner_{1NN}$ and $\learner_{lin}$ with the Open-Featuring and Error-Driven-Featuring  protocols.]
{
\begin{tabular}{c|cc|cc}
\label{tab:teachingcosts}
& \multicolumn{2}{|c|}{Open}& \multicolumn{2}{c}{Error-Driven}\\

Feat. Set & $\learner_{1NN}$ & $\learner_{lin}$& $\learner_{1NN}$ & $\learner_{lin}$\\
\hline
$F_1$ & (0,$\infty$) & (0,$\infty$) & (0,$\infty$) & (0,$\infty$) \\
$F_2$ & (1,3) & (1,$\infty$)& (1,3) & (1,$\infty$)  \\
$F_3$ & (2,7) & (2,$\infty$)& (2,7) & (2,$\infty$)  \\
$F_4$ & (2,2) & (2,3)& (2,2) & (2,3)  \\
\hline
$G_1$ & (0,$\infty$) & (0,$\infty$) & (0,$\infty$,) & (0,$\infty$)  \\
$G_2$ & (1,9) & (1,$\infty$) & (1,9) & (1,$\infty$) \\
$G_3$ & (1,8) & (1,$\infty$) & (1,8) & (1,$\infty$) \\
$G_4$ & (2,9) & (2,2) & (2,$\infty$) & (2,3) \\
\end{tabular}
}
\caption{Optimal feature set costs and optimal teaching costs for all of the feature sets from Figure~\ref{fig:TeachingEffort}.}
\label{tab:costs}
\end{table} 

The Open-Featuring protocol affords the teacher more flexibility than the Error-Driven-Featuring protocol. In particular, assuming that the teacher is an ideal teacher then there would be no reason to prefer the Error-Driven-Featuring protocol. If, however, the teacher is not an ideal teacher, one not always able to identify features that improve the representations or one who benefits from inspecting an invalidation set to identify features, then one might prefer the Error-Driven-Featuring protocol. In particular, this is a possibility that adding a poor feature can increase the labeling cost.  For instance, when using $\learner_{1NN}$,  a poor teacher who has taught the learner to use feature $f_1$ might add feature $f_2$ rather than feature $f_3$ significantly increasing the concept specification cost. In the next section we demonstrate that there is, in fact, unbounded risk for $\learner_{1NN}$. 

One of the short-comings of the Error-Drive-Featuring protocol is that, once the feature set is sufficient the teacher cannot add another feature.  For instance, for the example in Figure~\ref{fig:TeachingEffort1}, $F_3$ and $F_4$ are inaccessible. This might mean that representations that have lower concept specification costs cannot be used to teach $\target$. For instance, $F_4$ has a concept specification cost of 2 whereas the concept specification cost of $F_2$ is 3. While this difference is not large, it is easy to create an example where the costs differ significantly. In contrast, using the Open-Featuring protocol, a teacher can choose to teach either $F_2$ or $F_4$ trading of the costs of adding features and concept specification (adding labels).

The use of the Error-Driven-Featuring protocol can mitigate the risk of poor featuring but, as discussed above, does come with potential costs.  An alternative approach to mitigating the risks of featuring is to use a different learning algorithm. If we use $\learner_{lin}$, the potential for a increasing the cost for concept specification is when adding a feature is significantly limited. This is discussed in more detail in the next section. 

\section{Bounding Optimal Teaching Cost and Feature Set Costs}

In this section, we provide bounds on the optimal feature set teaching costs and optimal teaching costs for $\learner_{lin}$ and $\learner_{1NN}$ with the teaching protocols defined in Section~\ref{sec:protocols}. In this section, we assume that there is a finite set of realizable objects (i.e., $|X|<\infty$). 

\subsection{Bounding Optimal Feature Set Costs}

We provide a set of propositions each of which provides tight bounds for optimal concept specification costs and optimal invalidation costs for $\learner_{lin}$ and $\learner_{1NN}$. These propositions are presented in Table~\ref{tab:props} with their full statements with proofs presented in the full paper.

The fact that the optimal concept specification cost is unbounded as a function of the size of the feature set for $\learner_{1NN}$ is due to the fact that the 1NN classifier is of high capacity. Proposition~\ref{thm:LinConceptSpecBound}, however, bounds the potential increase in effort required to define the concept when adding a feature for $\learner_{lin}$. It is important to note that optimal concept specification cost for $\learner_{lin}$ can be just two labeled objects but not in general. In fact, one can construct for $d>1$, a set of objects and a feature set of size $d$ that requires $d+1$ objects to specify a linear hyperplane that generalizes to all of the objects.

Similar to the bound on the optimal concept specification cost, the bound optimal invalidation cost for $\learner_{lin}$ (Proposition~\ref{thm:LinInvalidBound}) is tight. This can be demonstrated by constructing, for $d\geq0$ a set of $d+2$ labeled objects in $\real^d$ such that any subset of the labeled objects is linearly separable. While Proposition~\ref{thm:LinInvalidBound} does provide a bound on the invalidation cost $\learner_{1NN}$, this bound for $\learner_{lin}$ is larger than that provided by Proposition~\ref{thm:1NNInvalidBound}. We suspect, however, that in practice, the invalidation cost for the linear classifier would typically be far less then $d+2$ for non-trivial $d$.

\begin{table}[ht]
\centering
\begin{tabular}{c|cc}
Algorithm & Concept Spec. Cost & Invalidation Cost  \\
\hline
$\learner_{lin}$ & $\leq |F|+1$ (Proposition~\ref{thm:LinConceptSpecBound}) & $\leq |F|+2$ (Proposition~\ref{thm:LinInvalidBound}) \\
$\learner_{1NN}$ & unbounded (Proposition~\ref{thm:1NNexplosion}) & =2 (Proposition~\ref{thm:1NNInvalidBound})  \\
\end{tabular}

\caption{Summary of propositions bounding the optimal invalidation cost and optimal concept specification cost for a feature set $F$ using $\learner_{lin}$ and $\learner_{1NN}$.}
\label{tab:props}
\end{table} 

\subsection{Bounding Teaching Costs}

In this section we consider bounding the cost of teaching a target classification function $\target$ using learning algorithms $\learner_{1NN}$ and $\learner_{lin}$. 

First we consider $\learner_{1NN}$.
Due to Proposition~\ref{thm:1NNexplosion}, we cannot bound the risk of adding a bad feature and thus cannot bound the teaching costs for our teaching protocols. We can, however, provide bounds for our teaching protocols using $\learner_{lin}$. The following proposition provides and upper bound on the teaching cost for a feature set.

\begin{proposition} \label{thm:opencost}
The labeling cost for a sufficient feature set $F$ using an optimal teacher and the Open-Featuring protocol with learning algorithm $\learner_{lin}$ is $\leq |F|+1$.
\end{proposition}

For the Error-driven-featuring protocol the computation of cost is more difficult as we need to account for the cost of invalidating feature sets. Proposition~\ref{thm:reuse} demonstrates a  useful connection between the invalidation sets for nested feature sets when using a linear classifier.

\begin{proposition}\label{thm:reuse}
If $T$ is an invalidation set for $F$, target classification function $\target$ and a consistent linear learner then $T$ is an invalidation set for $F'\subset F$.
\end{proposition}

Finally, the following proposition provides an upper bound on the teaching cost for a feature set for the learning algorithm $\learner_{lin}$.

\begin{proposition} \label{thm:errordrivencost}
The labeling cost for a minimal sufficient feature set $F$ using an optimal teacher and the Error-Driven-Featuring protocol with learning algorithm $\learner_{lin}$ is $\leq 2 (|F|+1)$.
\end{proposition}

\newpage

\bibliography{TeachingEffort}
\bibliographystyle{aaai}

\newpage
\section{Appendix}
In this section we provide proofs for Propositions. Several proofs  rely on convex geometry and we assume that the reader is familiar with basic concepts and elementary results from convex geometry. We denote the convex closure of a set of points by $conv(S)$. 

\noindent {\bf Proposition \ref{thm:insufficient}} {\em
If learning algorithm $\learner$ is consistent and $T$ is an invalidation set for feature set $F_i$, target concept $\target$, and $\learner$ then $F_i$ is not sufficient for $\target$ and $\learner$.

\begin{proof}
Let $T$ be an invalidation set for $F_i$, target concept $\target$ and consistent learning algorithm $\learner$. Aiming for a contradiction, we assume that $F_i$ is sufficient for $\target$ and $\learner$.
From the fact that $F_i$ is sufficient for target concept $\target$ and learning algorithm $\learner$ then there exists a training set $T'$ such that $\learner(F_i,T')=\target$. This implies that there is a classification function in the hypothesis class of the learning algorithm that is consistent with any (honest) training set including $T$. This fact and the fact that $T$ is an invalidation set implies $\learner$ is not consistent and we have a contradiction. It follows that $F_i$ is not sufficient.
\end{proof}

\noindent {\bf Proposition \ref{thm:sufficiency}} {\em
The learning algorithms $\learner_{1NN}$ and $\learner_{lin}$ are monotonically sufficient.
}

\begin{proof}
For $\learner_{1NN}$ we simply node that adding features makes more distinctions between objects thus once sufficient any superset will remain sufficient.

For $\learner_{lin}$, let $d$-dimensional feature set $F$ be sufficient for the target classification function. This means that $\exists (w,b)$ for $w\in \real^d$ and $b\in \real$ such that $\target(x)=sign(w \cdot F(x) + b)$. For $F'\supset F$ if we use an offset $b'=b$ and a weight vector $w'$ this agrees with $w$ for any feature $f\in F$ and is zero otherwise is equivalent to the classifier defined by $(w,b)$ (i.e., $sign(w'\cdot F'(x)+b) = sign(w \cdot F(x) + b)$) which proves the claim.
\end{proof}

\begin{lemma} \label{lem:margin}
If finite sets $S,T\subset \real^d$ that are strictly separable then there exists a subset $U\subseteq S \cup T$  such that $|U|\leq d+1$ and the maximum margin separating hyperplane defined by $U\cap S$ and $U\cap T$ separates $S$ and $T$.
\end{lemma}

\begin{proof}
We define the set of points that are the closest points in the convex closure of $S$ and $T$ (i.e., $CP(S,T)=\{(s,t)| s\in conv(S), t\in conv(T), \forall s'  \in conv(S) \forall t'  \in conv(T) dist(s,t)\leq dist(s',t' )\}$). The maximum margin hyperplane defined by any two points $(s,t)\in CP(S,T)$ suffice to define a hyperplane that separate $S,T$ (see, e.g., Liu et al 2016). Consider a pair $(s,t)\in CP(S,T)$. Due the the construction of the set it must be the case that $s$ belongs to some face of $conv(S)$ and similarly $t$ belongs to some face of $conv(T)$. In fact, the points are a subset of the Cartesian product a face of $conv(S)$ and a face of $conv(T)$ that share one or more points that are equidistant. 

Next we choose a subset of $CP(S,T)$ on the basis of the faces to which each of the pair of points belongs.
Let $dim⁡(x,U)$ be Euclidean dimension of the minimal face of $conv(U)$ containing $x$ or be $\infty$ if $x$ is not in a face of $conv(U)$. We define the minimal closest pairs (a subset of $CP(S,T)$) to be pairs whose summed face Euclidean dimension is minimal (i.e, $MinCP(S,T)=\{(s,t)\in CP(S,T) | \forall s' \in conv(S), \forall t' \in conv(T),  (s' ,t'  )\in CP(S,T)$ implies $dim⁡(s,S)+dim⁡(t,T)\leq dim⁡(s' ,S)+dim⁡(t' ,T) \}$

Let $(s,t)\in MinCP(S,T)$. Next we establish that $dim⁡(s,S)+dim⁡(t,T)\leq d-1$. Suppose this is not the case, that is, $d_s=dim⁡(s,S)$, $d_t=dim⁡(t,T)$ and $d_s+d_t \geq d$. In this case, consider the $d_s$  dimensional ball of variation around $s$ and the $d_t$ dimensional ball of variation around $t$. Because , $d_s+d_t\geq d$ there must be a parallel direction of variation. Rays in this direction starting at $s$ and $t$ define pairs of points in $CP(S,T)$. Following this common direction of variation from both $s$ and $t$ we must either hit a lower dimensional face of $conv(S)$ or $conv(T)$ which implies that $(s,t)\not \in MinCP(S,T)$. We have a contradiction and thus $d_s + d_t \leq d-1$.

Finally, if $dim⁡(s,S)+dim⁡(t,T)\leq d-1$ then by applying Carath{\'e}odory's theorem twice we can represent $s$ via $d_s+1$ point and $t$ via $d_t+1$ and thus $d+1$ points suffice to define a separating hyperplane for $S,T$ using a maximum margin hyperplane.
\end{proof}

\noindent {\bf Proposition \ref{thm:reuse}} {\em
If $T$ is an invalidation set for $F$, target classification function $\target$ and a consistent linear learner then $T$ is an invalidation set for $F'\subset F$.
}

\begin{proof}
Let $T$ be an invalidation set for $F$,$\target$, and consistent linear learner $\learner$. Suppose that $T$ is not an invalidation set for $F'$. In this case, there are parameters $(w',b')$ such that $c'(x)=sign(w' \cdot F'(x) + b')=\learner(F',T)$ is consistent with $T$. This means that there are parameters $(w,b)$ such that $c(x)=sign(w \cdot F(x) + b)$ is consistent with $T$ and thus $T$ is not an invalidation set for $F$ which is a contradiction. Thus $T$ must be an invalidation set for $F'$ proving the proposition.
\end{proof}

\noindent {\bf Proposition \ref{thm:opencost}} {\em
The labeling cost for a sufficient feature set $F$ using an optimal teacher and the Open-featuring protocol with learning algorithm $\learner_{lin}$ is upper-bounded by $|F|+1$.
}

\begin{proof}
Follows immediately from Proposition~\ref{thm:LinConceptSpecBound}.
\end{proof}

\noindent {\bf Proposition \ref{thm:errordrivencost}} 
{\em
The labeling cost for a minimal sufficient feature set $F$ using an optimal teacher and the Error-driven-featuring protocol with learning algorithm $\learner_{lin}$ is upper-bounded by $2 (|F|+1)$.
}

\begin{proof}
Consider the ideal teacher that first provides labels to invalidate subsets of $F$ along some path to $F$ in the feature lattice $\scriptF$ and then provides labels to teach the classification function. Because $F$ is minimally sufficient consider any subset $F'\in \scriptF$ such that $F'\subset F$ and $|F'|+1=|F|$. $F'$ is not sufficient and by Proposition~\ref{thm:LinInvalidBound} there is an invalidation set of size $|F| + 1$. Due to Proposition~\ref{thm:reuse} this  invalidation set is an invalidation set for all feature sets along paths in $\scriptF$ to $F'$ and thus the examples in this set are sufficient to allow the teacher to add the features in $F$. In the second phase, the teacher, by Proposition~\ref{thm:LinConceptSpecBound} need only provide at most $|F|+1$ additional labels to create a concept specification set. Thus, in the two phases, the optimal teacher need provide at most $2 (|F|+1)$ labeled examples.
\end{proof}

\begin{proposition} \label{thm:1NNexplosion}
Adding a single feature to a feature set can increase the concept specification cost variability (by $O(|X|)$) when using the 1NN learning algorithm.
\end{proposition}

\begin{proof}
The example configuration used in the feature set $F_3$ from the example from Figure~\ref{fig:TeachingEffort1} can be extended to arbitrarily many points.
\end{proof}

\begin{proposition}\label{thm:LinConceptSpecBound}
For any consistent linear learner, if a $d$-dimensional feature set $F$ is linearly sufficient for the target classification function then the concept specification cost is at most $d+1$.
\end{proposition}

\begin{proof} Let $X$ be our set of objects and $target$ be our target classification function. Define $S=\{F(x)\in \real^d| x\in X$ and $\target(x)=1\}$ and $T=\{F(x)\in \real^d | x\in X$ and $\target(x)=0\}$. Because $F$ is linearly sufficient then there exists a hyperplane separating the positive $X^+$ an negative examples $X^-$. We then apply Lemma~\ref{lem:margin} using $X^+$ and $X^-$ to obtain the desired result. 
\end{proof}

\begin{proposition}[Meek 2016]\label{thm:1NNInvalidBound}
If $F_i$ is not sufficient for the target classification function $\target$ using learning algorithm $\learner_{1NN}$ then the invalidation cost for feature set $F_i$ and $\learner_{1NN}$ is two.
\end{proposition}

\begin{proposition}[Meek 2016]\label{thm:LinInvalidBound}
For any consistent linear learner, if $d$-dimensional feature set $F$ is not linearly sufficient for the target classification function then the representation invalidation cost is at most $d+2$.
\end{proposition}

\end{document}